\documentclass[journal]{IEEEtran}
\usepackage{times}
\usepackage{soul}
\usepackage{url}
\usepackage[hidelinks]{hyperref}
\usepackage[utf8]{inputenc}
\usepackage[small]{caption}
\usepackage{graphicx}
\usepackage{amsmath}
\usepackage{amsthm}
\usepackage{booktabs}
\usepackage{algorithm}
\usepackage{algorithmic}
\usepackage{times}  
\usepackage{helvet} 
\usepackage{courier}  
\usepackage{caption}
\usepackage{mathtools}

\usepackage{color}
\newtheorem{theorem}{Theorem}
\newtheorem{lemma}{Lemma}

\usepackage{amssymb}
\usepackage{multirow} 
\usepackage{subfig}
\usepackage[utf8]{inputenc}

\ifCLASSINFOpdf
\else
\fi

\begin{document}
%
\title{Self-paced Principal Component Analysis }
%
%

\author{Zhao Kang, Hongfei Liu, Jiangxin Li, Xiaofeng Zhu, and Ling Tian
\thanks{The authors are with the School of Computer Science and Engineering,
University of Electronic Science and Technology of China, Chengdu, Sichuan,
611731. E-mail: zkang@uestc.edu.cn}
\thanks{Manuscript received April 19, 2005; revised August 26, 2015.}}

%
%

\markboth{Journal of \LaTeX\ Class Files,~Vol.~14, No.~8, August~2015}%
{Shell \MakeLowercase{\textit{et al.}}: Bare Demo of IEEEtran.cls for IEEE Journals}
%


\IEEEtitleabstractindextext{


\begin{abstract}
  Principal Component Analysis (PCA) has been widely used for dimensionality reduction and feature extraction. Robust PCA (RPCA), under different robust distance metrics, such as $\ell_1$-norm and $\ell_{2,p}$-norm, can deal with noise or outliers to some extent. However, real-world data may display structures that can not be fully captured by these simple functions. In addition, existing methods treat complex and simple samples equally. By contrast, a learning pattern typically adopted by human beings is to learn from simple to complex and less to more. Based on this principle, we propose a novel method called Self-paced PCA (SPCA) to further reduce the effect of noise and outliers.
  Notably, the complexity of each sample is calculated at the beginning of each iteration in order to integrate samples from simple to more complex into training. Based on an alternating optimization, SPCA finds an optimal projection matrix and filters out outliers iteratively. Theoretical analysis is presented to show the rationality of SPCA. Extensive experiments on popular data sets demonstrate that the proposed method can improve the state-of-the-art results considerably. 
\end{abstract}
}

\maketitle

\IEEEdisplaynontitleabstractindextext

\IEEEpeerreviewmaketitle
\vspace{0.6cm}

\IEEEraisesectionheading{\section{Introduction}\label{sec:introduction}}
\IEEEPARstart{N}{owadays}, machine learning, pattern recognition, and data mining applications often involve data with high-dimensionality, such as face images, videos, gene expressions, and time series. Directly analyzing such data will suffer from the curse of dimensionality and lead to suboptimal performance \cite{zhao2018new,liu2017efficient}. Therefore, it is paramount to look for a low-dimensional space before subsequent analysis. PCA is a popular technique for this task \cite{zhang2012trace,peng2020robust}.  

Basically, PCA seeks a projection matrix such that the projected data well reconstruct the original data in a least square sense, which inherently makes PCA sensitive to noise and outliers \cite{shahid2015robust,candes2011robust}. In practice, data are often contaminated. To handle this problem, a number of robust versions of PCA have been developed in the last few years. They can be broadly classified into two categories: $\ell_1$-norm based approaches and nuclear-norm based approaches. In essence, nuclear-norm based methods aim to find a clean data with low-rank structure \cite{peng2019res}. Generally, this kind of methods does not directly generate a lower dimension representation. Some representative methods are robust PCA (RPCA) \cite{candes2011robust}, graph-based RPCA \cite{shahid2015robust,kang2019robust}, and non-convex RPCA \cite{netrapalli2014non,kang2015robust}, which are typically used for foreground-background separation. Moreover, they are transductive methods and can not handle out-of-samples. Though Bao et al. \cite{bao2012inductive} propose an inductive approach, it is targeted for a clean data. Liu et al. \cite{9141402} develop a computationally simple paradigm for image denoising using superpixel-based PCA. Zhu et al. \cite{2017Graph} integrate PCA with manifold learning to learn the hash functions to achieve efficient similarity search.

Unlike nuclear-norm based methods, $\ell_1$-norm PCA adopts $\ell_1$-norm to replace squared Frobenius norm as the distance metric. For instance, $L_1$-PCA tries to minimize the $\ell_1$-norm reconstruction error \cite{ke2005robust}. Though it improves the robustness of PCA, it does not have rotational invariance \cite{ding2006r}. Some methods maximize $\ell_1$-norm covariance \cite{wang2014robust,ju2015image}. CS-$\ell_1$-PCA developed in \cite{2016Compressed} calculate robust subspace components by explicitly maximizing $\ell_1$ projection to enable low-latency video surveillance.

Some recent works point out that aforementioned $\ell_1$-norm PCA methods need to calculate the data mean in the least square sense, which is not optimal for non-Frobenius norm \cite{oh2016generalized,wang2017optimal}. Therefore, optimal mean RPCA (RPCA-OM) \cite{nie2014optimal} optimizes both projection matrix and mean. Nevertheless, it can achieve the global mean \cite{song2017low}. \cite{luo2016avoiding} maximizes the projected $\ell_1$ differences between each pair of points. Though it avoids the mean computation, it can be easily stuck into bad local minima. $\ell_{2,p} (0<p<2)$ is further used to measure the variation between each pair of projected data in $L_{2,p}$-RPCA \cite{liao2018robust}. 

Despite aforementioned approaches use different types of robust objective functions, real-world data might display structures that can not be fully captured by these fixed functions \cite{haeffele2019structured}. In addition, they have another inherent drawback, \emph{i.e.}, they treat complex and simple samples equally, which violates the human cognitive process. Human learning starts from simple instances of learning task, then introduces complex examples step by step. This learning scheme is called self-paced learning \cite{kumar2010self} and can alleviate the outliers issue \cite{zhang2018self,meng2015objective}. 

To improve the robustness of existing RPCA, we propose a novel method called Self-paced PCA (SPCA) by imitating human learning. Based on $L_{2,p}$-RPCA, we design a new objective function which evaluates the easiness of samples dynamically. Consequently, our model can learn from simple to more complicated samples. Both theoretical analysis and experimental results show that our new method is superior to prior robust PCA algorithms for dimensionality reduction. 

In summary, our main contributions are the following.
\begin{itemize}
    \item To further eliminate the impact of noise and outliers, we introduce cognitive principle of human beings into PCA. This can improve the generalization ability of PCA. Theoretical analysis reveals the robustness nature of our method.
 \item A novel weighting function is designed for maximization problem, which can define the complexity of samples and gradually learn from ``simple" samples to ``complex" samples in the learning process.
    \item  Both numerical and visual experimental results justify the effectiveness of our method.
\end{itemize}

\section{Principal Component Analysis Revisited}

Given a training data matrix  $X = [x_1, x_2, \dots, x_n]\in \mathcal{R}^{d\times n}$, where $d$ is number of features and $n$ is the number of samples. The goal of PCA is to find a low-rank projection matrix $U = [u_1, u_2, \dots, u_k]\in \mathcal{R}^{d\times k}$ that projects the feature space of original data to a new feature space with lower dimensionality $k(k<d)$ such that the reconstruction error is minimized. Mathematically, it solves
\begin{equation}
\label{pcaObjective}
\min_{U^{\top}U=I_k,\textbf{m}} \sum_i  \bigl\|(x_i - \textbf{m}) - UU^{\top}(x_i - \textbf{m})\bigr\|_2^2,
\end{equation}
where \textbf{m} denotes the mean of the training data and $I_k$ denotes a $k$-dimensional identity matrix.
The data mean \textbf{m} can be easily obtained by setting the derivative of Eq. \eqref{pcaObjective} with respect to \textbf{m} to zero, which yields $\textbf{m}= \dfrac{1}{n} \sum\limits_i x_i$. Problem \eqref{pcaObjective} can be equivalently formulated as
\begin{equation}
\label{pcaObjective2}
\max_{U^{\top}U=I_k} \sum_i  \bigl\|U^{\top}(x_i - \textbf{m})\bigr\|_2^2.
\end{equation}
As can be seen that the solutions of above objective functions are dominated by squared large distance, thus are greatly deviated from the real ones.

To improve the robustness of PCA, the $\ell_1$-norm based distance metric is applied \cite{kwak2008principal}, \emph{i.e.}, which solves 
\begin{equation}
\label{pcaObjective3}
    \max_{U^{\top}U=I_k} \sum_i \bigl\|U^{\top}(x_i - \textbf{m})\bigr\|_1.
\end{equation}
However, model \eqref{pcaObjective3} uses incorrect \textbf{m} which is estimated under the squared $\ell_2$-norm distance metric. To handle it, Nie \emph{et al.} \cite{nie2014optimal} integrate mean calculation in the objective function and propose RPCA-OM, \emph{i.e.}, 
\begin{equation}
    \label{pcaObjective4}
    \max_{U^{\top}U=I,\textbf{m} }\sum_i \bigl\|U^{\top}(x_i - \textbf{m})\bigr\|_2.
\end{equation}
Different from previous models, $U$ and \textbf{m} are alternatively updated in \eqref{pcaObjective4}. However, this usually incurs error accumulation. 

Luo \emph{et al.} \cite{luo2016avoiding} further propose RPCA-AOM to avoid mean computation, \emph{i.e.},
\begin{equation}
    \label{pcaObjective5}
    \max_{U^{\top}U=I_k} \sum_{i,j} \bigl\|U^{\top}(x_i - x_j)\bigr\|_1 
\end{equation}
Nevertheless, model \eqref{pcaObjective5} can not produce the optimal solution and lacks theoretical guarantee that $\ell_1$-norm relates to the covariance matrix \cite{liao2018robust}. To tackle this problem, Liao \emph{et al.} \cite{liao2018robust} adopt $\ell_{2,p}$-norm to characterize the geometric structure and propose $L_{2,p}$-RPCA 
\begin{equation}
    \label{pcaObjective6}
    \max_{U^{\top}U=I_k} \sum_{i,j} \bigl\|U^{\top}(x_i-x_j)\bigr\|_2^p.
\end{equation}

Though model \eqref{pcaObjective6} yields impressive performance, it is still potentially sensitive to noise and outliers in the presence of heavy noises and gross errors. In this paper, we propose to improve it by using a novel optimization strategy.

\section{Proposed Method}
\subsection{Motivation and Objective Function}
Human learning begins with easier instances of the task and gradually increase the difficulty level. Inspired by this, we expect to train PCA model in steps. To be precise, at the beginning, only ``simple" samples are included to train the model. Then, ``complex" samples are gradually fed into the model. This process alleviates the effects of noisy points and outliers, thus improves the model's generalization ability \cite{zhang2018self,meng2015objective}.

To accomplish this, we employ the self-paced learning technique. Generally, self-paced learning model is comprised of a fidelity term to measure the complexity of each sample and a regularizer term to impose penalty upon the weights of samples \cite{meng2015objective,guo2019adaptive,zhou2020self}. At the beginning, samples equipped with high weight are rare. As optimization algorithm goes on, the penalty of the regularizer increases and the number of samples with high weight automatically increases.
Therefore, we reach our SPCA model by combining self-paced learning with model \eqref{pcaObjective6}: 
\begin{equation}
    \label{ourObjective}
    \max_{U^{\top}U=I_k, w_i} \sum_{i,j=1}^{n} \bigl\|U^{\top}(x_i-x_j)\bigr\|_2^p w_i+f(w_i,\eta)
\end{equation}
where $w_i$ is the weight of $i_{th}$ sample and $f(w_i, \eta)$ is the regularizer with age parameter $\eta$ which will be explained in the next section. A number of regularizers have been developed in the literature. For instance, a recent one is \cite{jiang2018learn}
\begin{equation}
    f(w_i, \eta) = \eta(w_i\log w_i-w_i)
    \label{regularizer_min}
\end{equation}
where the optimal $w_i$ obtained by taking the derivative of Eq. \eqref{regularizer_min} is an decreasing function w.r.t. loss. Unlike previous methods \cite{meng2015objective,guo2019adaptive,zhou2020self,jiang2018learn}, where regularizers are suitable for minimization problem, our regularizer must be designed for maximization problem. The optimal $w_i$ should increase as the fidelity increases and finally converges to 1 as the fidelity value approaches infinity. To satisfy this property, we design a new regularizer as following:
\begin{equation}
    \begin{aligned}
        f(w_i, \eta) = -& \log(w_i+e^{-1/\eta})^{(w_i+e^{-1/\eta})}\\
                   -& \log(1-w_i)^{(1-w_i)} - \frac{w_i}{\eta}.
    \end{aligned}
    \label{f(w,eta)}
\end{equation}
The rationale behind it is discussed in the next subsection where we discuss how to update the weight $w_i$. 

\begin{figure}[!htbp]
\centering
\includegraphics[width=0.3\textwidth]{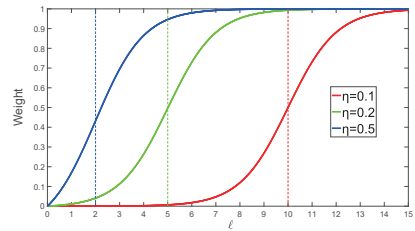} 
\caption{The change of weight $w^*$ with respect to the age parameter $\eta$ and fidelity $\ell$. The dashed line represents $\ell=1/\eta$. Note that we are maximizing our fidelity term $\ell$ so that the weight is an increasing function w.r.t. fidelity value.}
\label{learncurve}
\end{figure}

\subsection{Optimization}
To solve problem \eqref{ourObjective}, we adopt an alternative optimization strategy (AOS), \emph{i.e.}, we iteratively update one parameter while keeping the others fixed.

\textbf{Update $w_i$ for each sample}: 
The optimal weight of the $i$-th sample can be obtained by solving:  
\begin{equation}
\label{wlf}
\max_{w_i\in[0,1]} w_i\ell_i+f(w_i,\eta).
\end{equation}
where the fidelity term is denoted as $\ell_i$, i.e.,
\begin{equation}
\label{loss}
\ell_i= \sum_{j=1}^{n} \bigl\|U^{\top}(x_i-x_j)\bigr\|_2^p .   
\end{equation}

By setting the first-order partial derivative of \eqref{wlf} w.r.t. $w_i$ to 0, we obtain the closed-form solution of $w_i$:
\begin{equation}
\label{updateWi}
w_i^*= \frac{e^{\ell_i-1/\eta}-e^{-1/\eta}}{1+e^{\ell_i-1/\eta}}
\end{equation}
Ignoring the subscripts, we denote $w^*= \dfrac{e^{\ell-1/\eta}-e^{-1/\eta}}{1+e^{\ell-1/\eta}}$, which is a smooth function related to $\ell$. The $w^*$ values under different $\eta$’s are presented in Figure \ref{learncurve}. Given a fixed $\eta$, we obtain $\ell=1/\eta$ by letting the second-order partial derivative of $w^*$ w.r.t. $\ell$ be 0, which indicates that it is reasonable to regard $\ell=1/\eta$ as a threshold
to distinguish ``simple" and ``complex" samples. For those samples with $\ell$ larger than $1/\eta$, the growth of weight become slower w.r.t. $\ell$, hence can be implicitly regarded as ``simple". Otherwise, those samples with $\ell$ less than $1/\eta$ are considered as ``complex" samples.
$w^*$ is monotonically increasing with respect to $\ell$ and it holds that $\lim_{\ell \to 0} w^* = 0$ and $\lim_{\ell \to \infty} w^*=1$, which suggests that ``simple" samples are often preferred by the model because of their larger fidelity values. 
Additionally, it can be seen that when a smaller $1/\eta$ (larger $\eta$) is applied, the growth rate of weights of the complex samples becomes faster, which means that more samples tend to be included in the training process.

\textbf{Update $\ell_i$ for each sample}:
Note that we should update $\ell_i$ before $w_i$ since $w_i$ is dependent on $\ell_i$. We know from Figure \ref{learncurve} that given a fixed $\eta$, $w^*(\ell,\eta)$ has a distinct change only with $\ell$ in a specific interval where $w^*(\ell,\eta)$ increases rapidly. For instance, $w^*(\ell,\eta)$ varies considerably with $\ell$ in the range $[0,10]$ for $\eta=0.2$. In order to effectively distinguish ``simple" samples from ``complex" samples, we normalize the fidelity value of each sample to this ``most varying" range as follows:
\begin{equation}
\label{normalloss}
 \ell_i \coloneqq \frac{c \cdot \ell_i}{\max \{\ell_1, \ell_2, \dots, \ell_n\}}    
\end{equation}
where $c$ represents the normalizing coefficient. 

\textbf{Update $U$}:
 Ignoring term $f(w_i,\eta)$ which is a constant w.r.t. $U$, we have the following subproblem:
\begin{equation}
    \label{alternating}
   \max_{U^{\top}U=I_k} \sum_{i,j} \bigl\|U^{\top}(x_i-x_j)\bigr\|_2^p w_i.
\end{equation}
Define $s_{ij} = \bigl\|U^{\top}(x_i-x_j)\bigr\|_2^{p-2}$, above objective function has the following equivalent formulation: 
\begin{equation}
\label{Induce}
\begin{aligned}
    & \sum_{i,j} \bigl\|U^{\top}(x_i-x_j)\bigr\|_2^p w_i= \sum_{i,j}\bigl\|U^{\top}(x_i-x_j)\bigr\|_2^2 s_{ij} w_i \\
    &= \sum_{i,j}tr\Bigl(\bigl(U^{\top}(x_i-x_j)\bigr)\bigl(U^{\top}(x_i-x_j)\bigr)^{\top}\Bigr)s_{ij} w_i\\
    &= 2tr(U^{\top}XLX^{\top}U)= 2tr(U^{\top}H) \\
\end{aligned}
\end{equation}
where $L = D-S$, $H = XLX^{\top}U$, and $D$ is a diagonal matrix with diagonal elements $d_{ii} = w_i \sum\limits_j s_{ij}$.

By replacing the objective function in \eqref{alternating} with Eq. \eqref{Induce}, the problem turns into
\begin{equation}
    \label{trace}
    \mathop{\max}_{U^{\top}U=I_k}\; tr(U^{\top}H)
\end{equation}
Note that $H$ is dependent upon the projection matrix $U$, which means that Eq. \eqref{trace} cannot be easily solved. To address it, we follow an alternative strategy, i.e., we solve $U$ first by fixing $H$ and then update the value of $H$ using the new $U$. 


\begin{theorem}
\label{theo3}
Denote the compact singular value decomposition (SVD) of $A \in R^{m\times n}$ as $U\Sigma V^{\top}$ (\textit{with} $V^{\top}V=U^{\top}U=I_k, k=rank(A)$), 
    \begin{equation}
        \mathop{\arg\max}_{W^{\top}W=I_k}\; tr(W^{\top}A)    
    \end{equation}
    is $W=UV^{\top}$.
\end{theorem}
\begin{proof}
\begin{equation}
\begin{split}
tr(W^{\top}A)=&tr(W^{\top}U\Sigma V^{\top})\\
=&tr(U\Sigma^{1/2}\Sigma^{1/2}V^{\top}W^{\top})
\end{split}
\end{equation}
Based on $tr(X^{\top}Y) \leq \|X\|_F \|Y\|_F $, we have
\begin{equation}
\begin{split}
tr(W^{\top}A)=&\|U\Sigma^{1/2}\|_F\|\Sigma^{1/2}V^{\top}W^{\top}\|_F\\
=&\|\Sigma^{1/2}\|_F\|\Sigma^{1/2}\|_F
\end{split}
\end{equation}
Equality holds only if $\Sigma^{1/2}U^{\top}=\Sigma^{1/2}V^{\top}W^{\top}$, i.e., $W=UV^{\top}$.
    \end{proof}
According to theorem \ref{theo3}, the optimal solution of the objective function \eqref{trace} is
\begin{equation}
\label{solution}
    U = QV^{\top},
\end{equation}
where $Q$ and $V$ are from the SVD of $H$,\emph{ i.e.}, $H=Q\Sigma V^{\top}$, indicating that the solution to problem \eqref{ourObjective} incorporates the geometric structure of data $X$ since $H=XLX^{\top}U$ and $XLX^{\top}$ is an adaptive weighted covariance matrix. The pseudo code for solving problem \eqref{ourObjective} is summarized in the following algorithm. A maximum iteration number 10 is used as the stopping criterion.

\begin{algorithm}[!htbp]
    \caption{Algorithm to solve problem \eqref{ourObjective}}
    \begin{algorithmic}[1]
        \REQUIRE ~~\\
        1. Dataset $ \{x_i\in \mathcal{R}^d : i = 1, 2, \dots, n \}$, where $x_i$ is normalized; Predefined subspace dimension $k$; Predefined parameter $p$ and $\eta$;
        \ENSURE $U \in \mathcal{R}^{d\times k}$; 
        \STATE Initialize $U$ which satisfies $U^{\top}U = I_k$;
        \WHILE{not converge}
            \STATE Update $\ell$ with Eq. \eqref{loss};
            \STATE Normalize $\ell$ according to \eqref{normalloss};
            \STATE Update $w_i$ with Eq. \eqref{updateWi};
                \WHILE {not converge}
                    \STATE Calculate $H = XLX^{\top}U$;
                    \STATE Calculate the SVD of matrix $H$ by $H = Q\Sigma V^{\top}$;
                   \STATE Calculate $U = QV^{\top}$;
                \ENDWHILE
        \ENDWHILE
    \end{algorithmic}
\end{algorithm}

\subsection{Theoretical Analysis}
In this subsection, some important properties of SPCA are provided by analyzing our optimization strategy and the objective function.
We define $F_{\eta}(\ell)$ as the integration of $w^*(\ell,\eta)$ with respect to $\ell$:
\begin{equation}
\label{F}
    F_{\eta}(\ell) = \int_0^{\ell} w^*(l,\eta)dl
\end{equation}

Then we define $Q_{\eta}(U|U^*)$ as the first-order expansion of $F_{\eta}(\ell(U))$ at $\ell(U^*)$:
\begin{equation}
\label{Q}
    Q_{\eta}(U|U^*) = F_{\eta}(\ell(U^*)) + w^*(\ell(U^*),\eta)(\ell(U)-\ell(U^*))
\end{equation}

\begin{lemma}
\label{f<q}
    Given $\eta$ and $U^*$, the following inequality holds 
    \begin{equation}
        F_{\eta}(\ell(U)) \geq Q_{\eta}(U|U^*).
    \end{equation}
\end{lemma}

 \begin{proof}
     Since $w^*(\ell, \eta)$ is continuous with respect to $\ell$ on $[0,\infty)$, based on Eq. \eqref{F}, we have
    \begin{equation}
       w^*(\ell, \eta) = F_{\eta}^{'}(\ell)
   \end{equation}
    By Eq. \eqref{updateWi}, $w^*(\ell, \eta)$ is monotonically increasing with respect to $\ell$ on $[0,\infty)$. Thus, $F_{\eta}(\ell)$ is convex on $[0,\infty)$. Therefore, the first order Taylor expansion of $F_{\eta}(\ell)$ at $\ell(U^*)$, namely $Q_{\eta}(U|U^*)$, forms a lower bound of $F_{\eta}(\ell)$. Thus, we obtain
    \begin{equation*}
    \begin{split}
       F_{\eta}(\ell(U)) \geq& F_{\eta}(\ell(U^*)) + w^*(\ell(U^*),\eta)(\ell(U)-\ell(U^*))\\
     &=Q_{\eta}(U|U^*)
        \end{split}
   \end{equation*}
   \end{proof}

Though there are two variables in our objective function \eqref{ourObjective}, we show that SPCA simply maximizes an implicit objective function where $w$ completely disappears.
\begin{theorem}
\label{AOS}
    Given fixed $\eta$, the alternative optimization strategy (AOS) to maximize Eq. \eqref{ourObjective} is equivalent to the minorize-maximization (MM) \cite{lange2000optimization} algorithm for solving \begin{equation}
    \label{surrogate}
        \max \sum_{i=1}^n F_{\eta}(\ell_i (U)).
    \end{equation}
    
\end{theorem}

\begin{proof}
    Denote $U^{t}$ as the projection matrix in the $t^{th}$ iteration of the AOS for solving \eqref{ourObjective}. By the standard MM algorithm, we obtain the optimization step as follows:
    
    \textbf{Minorize step:} Based on Lemma \ref{f<q}, $Q_{\eta}^{i}(U|U^{t})$ is a surrogate function of $F_{\eta}(\ell_i(U^{t}))$ for problem \eqref{surrogate}. It is easy to see that for fixed $U^{t}$, $Q_{\eta}^{i}(U|U^{t})$ only depends on $w^*(\ell_i(U^{t}), \eta)$. To obtain $Q_{\eta}^{i}(U|U^{t})$, we calculate $w^*(\ell_i(U^{t}), \eta)$ by maximizing $F_{\eta}(\ell_i(U^{t}))$:
    \begin{equation*}
    \begin{aligned}
        w^*(\ell_i(U^{t}), \eta)&=\mathop{\arg\max}F_{\eta}(\ell_i(U^{t}))
        \\&=\mathop{\arg\max}_{w_i \in [0,1]} w_i\ell_i(U^{t})+f(w_i,\eta),
    \end{aligned}
    \end{equation*}
    which is the same as the first step of AOS that updates $w$ under fixed $U$.
    
    \textbf{Maximize step:} We update $U$ under surrogate function $Q_{\eta}^{i}(U|U^{t})$ by
    \begin{equation*}
    \begin{aligned}
       U^{t+1}&=\mathop{\arg\max}_{U}\sum_{i=1}^{n} Q_{\eta}^{i}(U|U^{t})
       \\
       &=\!\mathop{\arg\max}_{U}\sum_{i=1}^{n}F_{\eta}(\ell_i(U^{t}))\! +\! w^*(\ell_i(U^{t}),\eta)(\ell_i(U)\!-\!\ell_i(U^{t}))\!  \\
                &= \mathop{\arg\max}_{U} \sum_{i=1}^{n} w^*(\ell_i(U^{t}),\eta)\ell_i(U),
    \end{aligned}
    \end{equation*}
    which exactly corresponds to the second step of AOS in updating $U$ under fixed $w$.
\end{proof}

Therefore, the AOS used in our algorithm is actually equivalent to the well-known MM algorithm. Then, many conclusions of MM theory hold true for AOS. For instance, as iterations progress, the upper-bounded objective value of our model is monotonically increasing, which ensures the convergence of our algorithm. Moreover, with the surrogate function $F_{\eta}(\ell)$, we show in the following theorem that our model is robust to complex samples. 
\begin{theorem}
    \label{robustness}
    Suppose that $\max_k \ell_k<M$ and $M<\infty$, for any pair of different samples $(i,j)$ in training dataset:
    \begin{equation}
        \Bigl|F_{\eta}(\ell_i)-F_{\eta}(\ell_j)\Bigr| \leq w(M,\eta)\Bigl|\ell_i-\ell_j\Bigr|.
    \end{equation}
\end{theorem}

\begin{proof}
    Let $a=\min \{\ell_i, \ell_j\}$, $b=\max \{\ell_i, \ell_j\}$. By Lagrange's mean value theorem, we have
    \begin{equation}
        F_{\eta}(\ell_i)-F_{\eta}(\ell_j) = \frac{\partial F_{\eta}(\ell)}{\partial \ell}\Bigr|_{\ell=\xi}(\ell_i-\ell_j),
    \end{equation}
    where $\xi \in [a,b]$.
    Thus, we further deduce that
    \begin{equation}
        \begin{split}
            \Bigl|F_{\eta}(\ell_i)-F_{\eta}(\ell_j)\Bigr|& \leq 
                  \bigr(\mathop{sup}_{\ell \in [a,b]} \bigl|w(\ell,\eta)\bigr|\bigr)\Bigl|\ell_i-\ell_j\Bigr|\\
            &    \leq w(M,\eta)\Bigl|\ell_i-\ell_j\Bigr|.
                  \end{split}
    \end{equation}
\end{proof}
Theorem \ref{robustness} indicates that $F_{\eta}(\ell(\cdot))$ is more robust than original $\ell(\cdot)$ to complex instances with small $\ell$ value. We show this by comparing the fidelity difference between two samples $i$ and $j$, where $i$ is a complex sample and $j$ is a simple one. Since $w(M,\eta)<1$, the difference $\Bigl| F_{\eta}(\ell_i) - F_{\eta}(\ell_j)\Bigr|$ in SPCA is smaller than the original fidelity difference $\Bigl|\ell_i-\ell_j\Bigr|$. Hence, we can see that $F_{\eta}(\ell(\cdot))$ is less sensitive toward complex samples and SPCA less prone to overfit noised data points.
\begin{figure}[!htbp]
\centering
\includegraphics[width=0.4\textwidth]{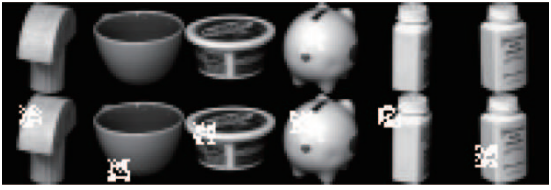} 
\caption{Samples selected from the COIL20 dataset. Original images are shown in the first row while noised images in the second row.}
\label{coil}
\end{figure}

\begin{figure}[!htbp]
\centering
\includegraphics[width=0.4\textwidth]{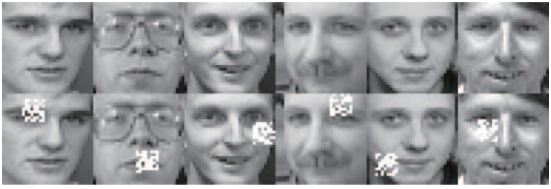} 
\caption{Some samples from the ORL dataset. The second row is noised images. }
\label{orl}
\end{figure}

\begin{figure}[!htbp]
\centering
\includegraphics[width=0.4\textwidth]{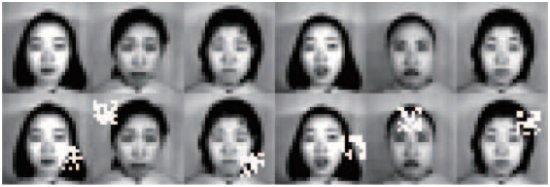} 
\caption{Some samples from the JAFFE dataset. The second row is noised images.}
\label{jaffe}
\end{figure}

\section{Experiments}

In this section, we compare the reconstruction error, reconstructed image and eigenfaces of our method SPCA ($p=$0.5, 1, and 1.5) with optimal mean robust PCA (RPCA-OM) \cite{nie2014optimal}, avoiding optimal mean robust PCA (RPCA-AOM) \cite{luo2016avoiding}, avoiding mean calculation with $\ell_{2,p}$-norm robust PCA ($L_{2,p}$-RPCA) \cite{liao2018robust}, and $\ell_1$-norm based RPCA via bit flipping ($L_1$-PCA) \cite{markopoulos2016l1}. Then we analyze the effect of the parameters $\eta$ and $p$ in our proposed approach and visualize the weight of each sample as our algorithm iterates.

\subsection{Experimental Setup}
Three databases (COIL20, ORL, and JAFFE) are utilized for the experiments. Specifically, the COIL20 database includes 1440 gray-scale images of 20 objects (72 images per object) \cite{nene1996columbia}. Each object is placed at the center of a mechanical turntable that is then rotated to vary the angel of object with respect to a fixed camera. The ORL database contains 400 images of 40 distinct subjects with the resolution 112x92 \cite{Samaria1994ParameterisationOA}. The images are taken at different conditions, such as facial expressions (open/closed eyes, smiling), facial details (glasses), and lighting, against a dark homogeneous background with the individuals standing upright. The JAFFE dataset consists of 213 images of 7 facial expressions (6 basic facial expressions + 1 neutral) posed by 10 Japanese female models \cite{lyons1998japanese}. 

Following the previous work \cite{luo2016avoiding,liao2018robust}, data points are all normalized and 30\% of them are randomly selected and placed a 1/4 side length square occlusion at random positions.
Some sample images are shown in Figures \ref{coil}-\ref{jaffe}. We randomly select half of the images from each class as the training data, then the remaining ones are used for testing. We fix $\eta=0.1$ for our method and set the corresponding normalizing coefficient $c=15$ for Eq. \eqref{normalloss}.
We first use the widely used metric, i.e., the average reconstruction error, to evaluate the dimension reduction effect: 
\begin{equation}
    \label{error}
    e = \frac{1}{n} \sum_{i=1}^{n} \|x_i^{clean} - UU^{\top}x_i^{clean}\|_2
\end{equation}
where $n$ is the number of testing images, $x_i^{clean}$ is the $i$-th clean testing image.


\begin{table*}[htbp]
  \centering
  \renewcommand{\arraystretch}{1.1}
  \caption{ Reconstruction error versus different dimensions of five methods on three databases. The best reconstruction result under each dimension is bolded.  }\smallskip
  \resizebox{0.95\textwidth}{!}{
    \begin{tabular}{|c|c|c|c|c|c|c|c|c|c|c|}
    \hline
    \multirow{10}{*}{JAFFE} & Dimension & 10    & 15    & 20    & 25    & 30    & 35    & 40    & 45    & 50 \\
\cline{2-11}          & RPCA-OM & 0.4535  & 0.4735  & 0.4649  & 0.4550  & 0.4157  & 0.3607  & 0.4251  & 0.4233  & 0.3990  \\
\cline{2-11}          & RPCA-AOM & 0.4672  & 0.4734  & 0.4591  & 0.4529  & 0.4170  & 0.3572  & 0.4075  & 0.4263  & 0.3979  \\
\cline{2-11}          & $L_{2,p}$-RPCA($p$=0.5) & 0.4708  & 0.4637  & 0.4416  & 0.4284  & 0.4004  & 0.3592  & 0.4014  & 0.4267  & 0.4010  \\
\cline{2-11}          & $L_{2,p}$-RPCA($p$=1.0) & 0.4745  & 0.4697  & 0.4623  & 0.4499  & 0.4164  & 0.3585  & 0.4260  & 0.4279  & 0.4016  \\
\cline{2-11}          & $L_{2,p}$-RPCA($p$=1.5) & 0.4755  & 0.4717  & 0.4631  & 0.4514  & 0.4172  & 0.3589  & 0.4133  & 0.4293  & 0.4032  \\
\cline{2-11}          & $L_1$-PCA & 0.5825  & 0.4920  & 0.5511  & 0.5363  & 0.5363  & 0.4773  & 0.5373  & 0.4996  & 0.5216  \\
\cline{2-11}          & SPCA($p$=0.5) & 0.4098  & 0.4180  & 0.4077  & 0.3998  & \textbf{0.3503 } & \textbf{0.2451 } & 0.3682  & \textbf{0.3426 } & \textbf{0.2532 } \\
\cline{2-11}          & SPCA($p$=1.0) & \textbf{0.4073 } & \textbf{0.4173 } & 0.4072  & \textbf{0.3981 } & 0.3584  & 0.2978  & 0.3629  & 0.3612  & 0.3250  \\
\cline{2-11}          & SPCA($p$=1.5) & 0.4238  & 0.4405  & \textbf{0.4067 } & 0.4054  & 0.3594  & 0.3026  & \textbf{0.3595 } & 0.3567  & 0.3384  \\
    \hline
    \multirow{10}{*}{ORL} & Dimension & 10    & 15    & 20    & 25    & 30    & 35    & 40    & 45    & 50 \\
\cline{2-11}          & RPCA-OM & 0.4709  & 0.4127  & 0.4165  & 0.3898  & 0.3944  & 0.4284  & 0.3971  & 0.3985  & 0.3632  \\
\cline{2-11}          & RPCA-AOM & 0.4145  & 0.4171  & 0.4335  & 0.4002  & 0.4066  & 0.4260  & 0.3970  & 0.3949  & 0.3555  \\
\cline{2-11}          & $L_{2,p}$-RPCA($p$=0.5) & 0.4843  & 0.4172  & 0.4140  & 0.3745  & 0.3794  & 0.3885  & 0.3482  & 0.3577  & 0.3241  \\
\cline{2-11}          & $L_{2,p}$-RPCA($p$=1.0) & 0.4520  & 0.4142  & 0.4123  & 0.3816  & 0.4064  & 0.4216  & 0.3915  & 0.3938  & 0.3573  \\
\cline{2-11}          & $L_{2,p}$-RPCA($p$=1.5) & 0.4480  & 0.4103  & 0.4269  & 0.3878  & 0.4045  & 0.4243  & 0.3958  & 0.3983  & 0.3617  \\
\cline{2-11}          & $L_1$-PCA & 0.5107  & 0.3979  & 0.4980  & 0.4019  & 0.4383  & 0.4837  & 0.4948  & 0.5053  & 0.4809  \\
\cline{2-11}          & SPCA($p$=0.5) & 0.3951  & 0.3685  & 0.3692  & 0.3317  & 0.3359  & 0.3665  & 0.3215  & 0.3304  & 0.3022  \\
\cline{2-11}          & SPCA($p$=1.0) & \textbf{0.3042 } & 0.2972  & 0.3096  & 0.2846  & 0.2979  & 0.3272  & 0.2958  & 0.2985  & 0.2690  \\
\cline{2-11}          & SPCA($p$=1.5) & 0.3269  & \textbf{0.2928 } & \textbf{0.3005 } & \textbf{0.2794 } & \textbf{0.2899 } & \textbf{0.3148 } & \textbf{0.2900 } & \textbf{0.2904 } & \textbf{0.2628 } \\
    \hline
    \multirow{10}{*}{COIL20} & Dimension & 10    & 15    & 20    & 25    & 30    & 35    & 40    & 45    & 50 \\
\cline{2-11}          & RPCA-OM & 3.2359  & 2.8202  & 2.6584  & 2.4717  & 2.2741  & 2.1419  & 1.9846  & 2.0064  & 1.9162  \\
\cline{2-11}          & RPCA-AOM & 3.1500  & 2.7807  & 2.6277  & 2.4729  & 2.2768  & 2.1349  & 2.0354  & 2.0147  & 1.9644  \\
\cline{2-11}          & $L_{2,p}$-RPCA($p$=0.5) & 3.0782  & 2.7657  & 2.6256  & 2.4392  & 2.2392  & 2.1405  & 2.0067  & 1.9979  & 1.9351  \\
\cline{2-11}          & $L_{2,p}$-RPCA($p$=1.0) & 3.0458  & 2.7683  & 2.6300  & 2.4512  & 2.2745  & 2.1391  & 2.0135  & 2.0091  & 1.9442  \\
\cline{2-11}          & $L_{2,p}$-RPCA($p$=1.5) & 3.1345  & 2.7823  & 2.6410  & 2.4441  & 2.2898  & 2.1444  & 2.0190  & 2.0148  & 1.9458  \\
\cline{2-11}          & $L_1$-PCA & 4.3870  & 4.4498  & 4.4861  & 4.4874  & 4.4678  & 4.3207  & 4.4206  & 4.1474  & 4.4257  \\
\cline{2-11}          & SPCA($p$=0.5) & 3.0844  & 2.7577  & 2.5985  & 2.4440  & 2.2540  & 2.1322  & 2.0021  & 2.0002  & 1.9236  \\
\cline{2-11}          & SPCA($p$=1.0) & 3.0294  & 2.7126  & 2.5391  & 2.3961  & 2.2100  & 2.0919  & \textbf{1.9693 } & 2.0044  & 1.9175  \\
\cline{2-11}          & SPCA($p$=1.5) & \textbf{3.0124 } & \textbf{2.6657 } & \textbf{2.4885 } & \textbf{2.3721 } & \textbf{2.1957 } & \textbf{2.0625 }   & 1.9726  & \textbf{1.9714} &\textbf{1.8974 } \\
    \hline
    \end{tabular}%
}    
  \label{main_result}%
\end{table*}

\subsection{Results}




\begin{figure*}[!htbp]
\centering
\includegraphics[width=0.3\textwidth]{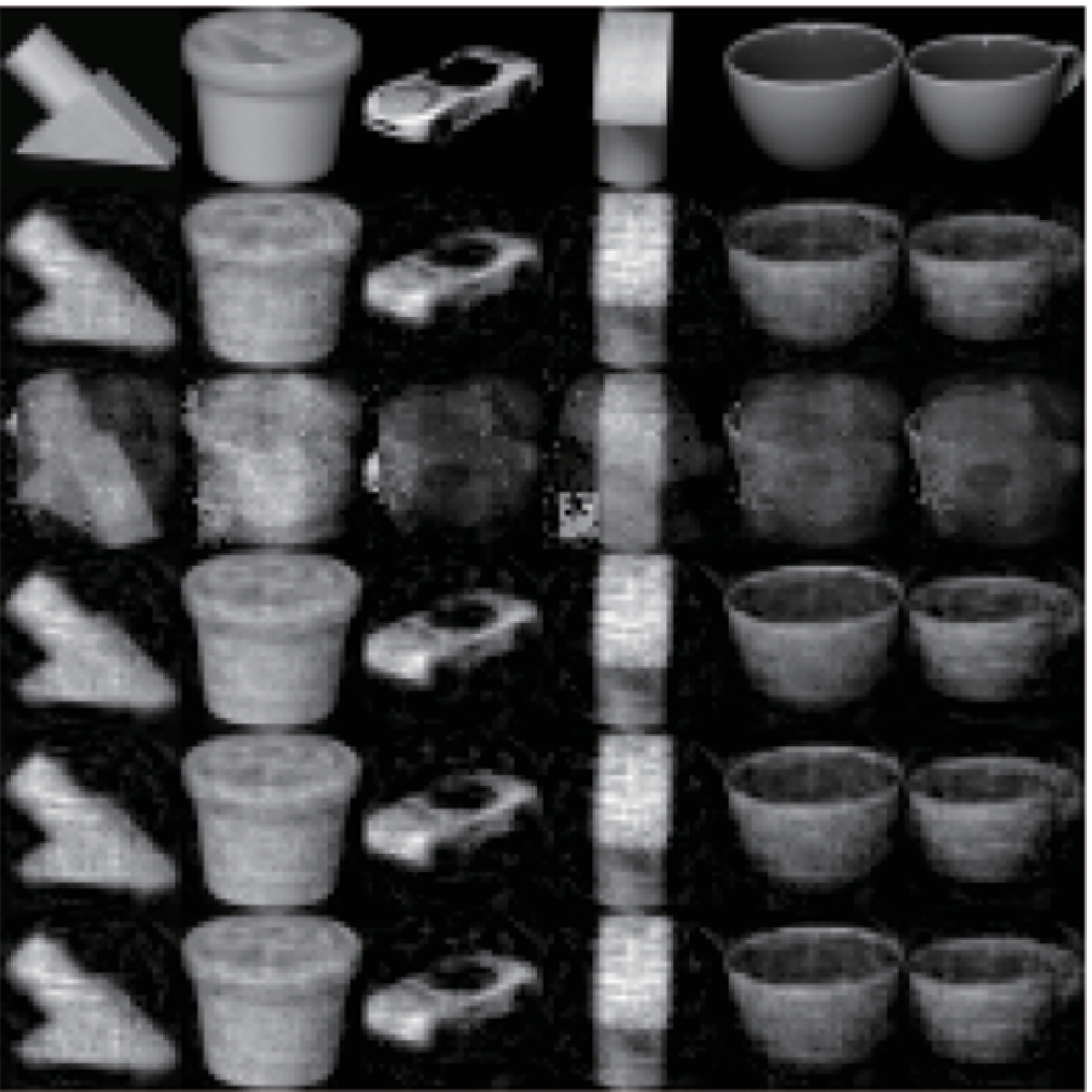}
\includegraphics[width=0.3\textwidth]{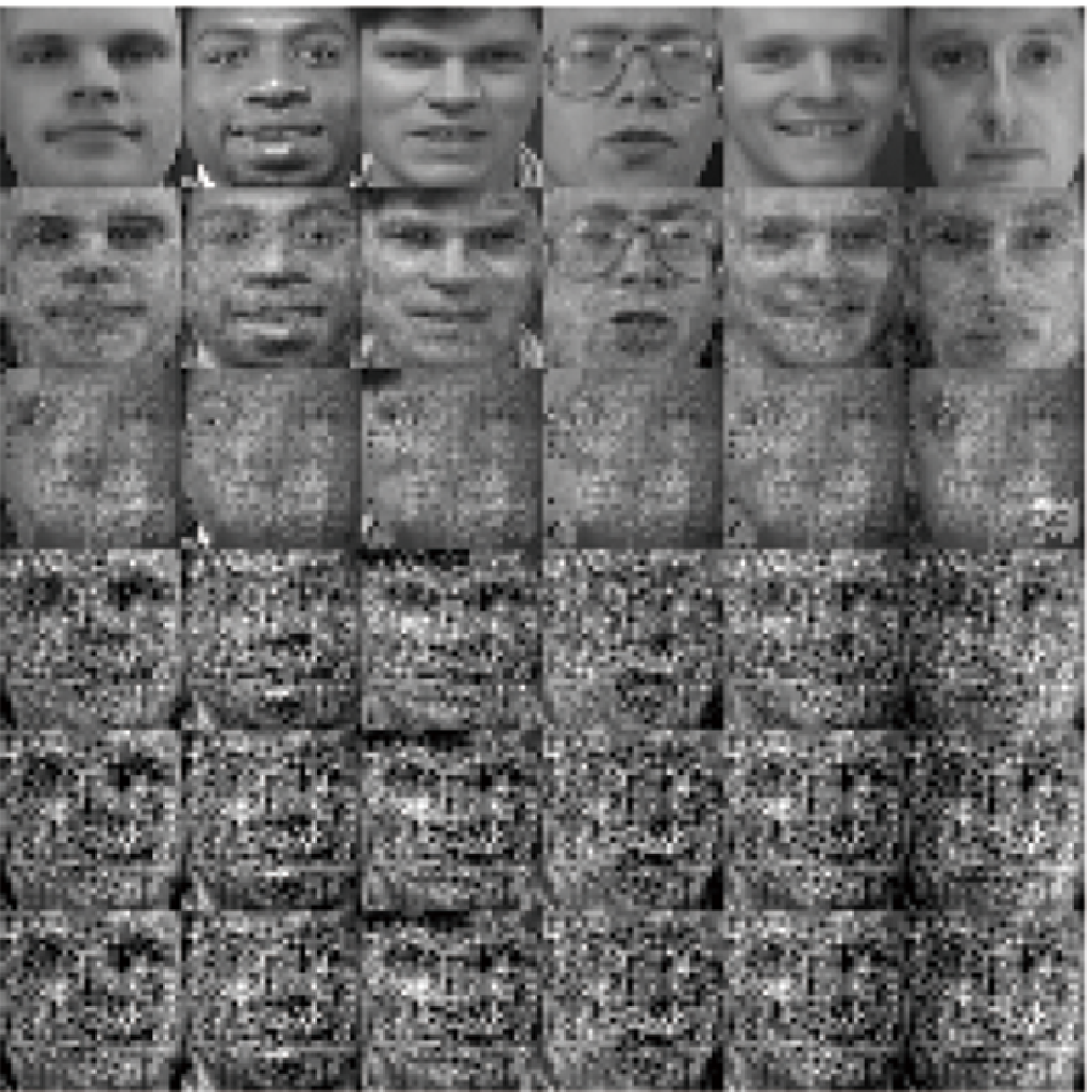}
\includegraphics[width=0.3\textwidth]{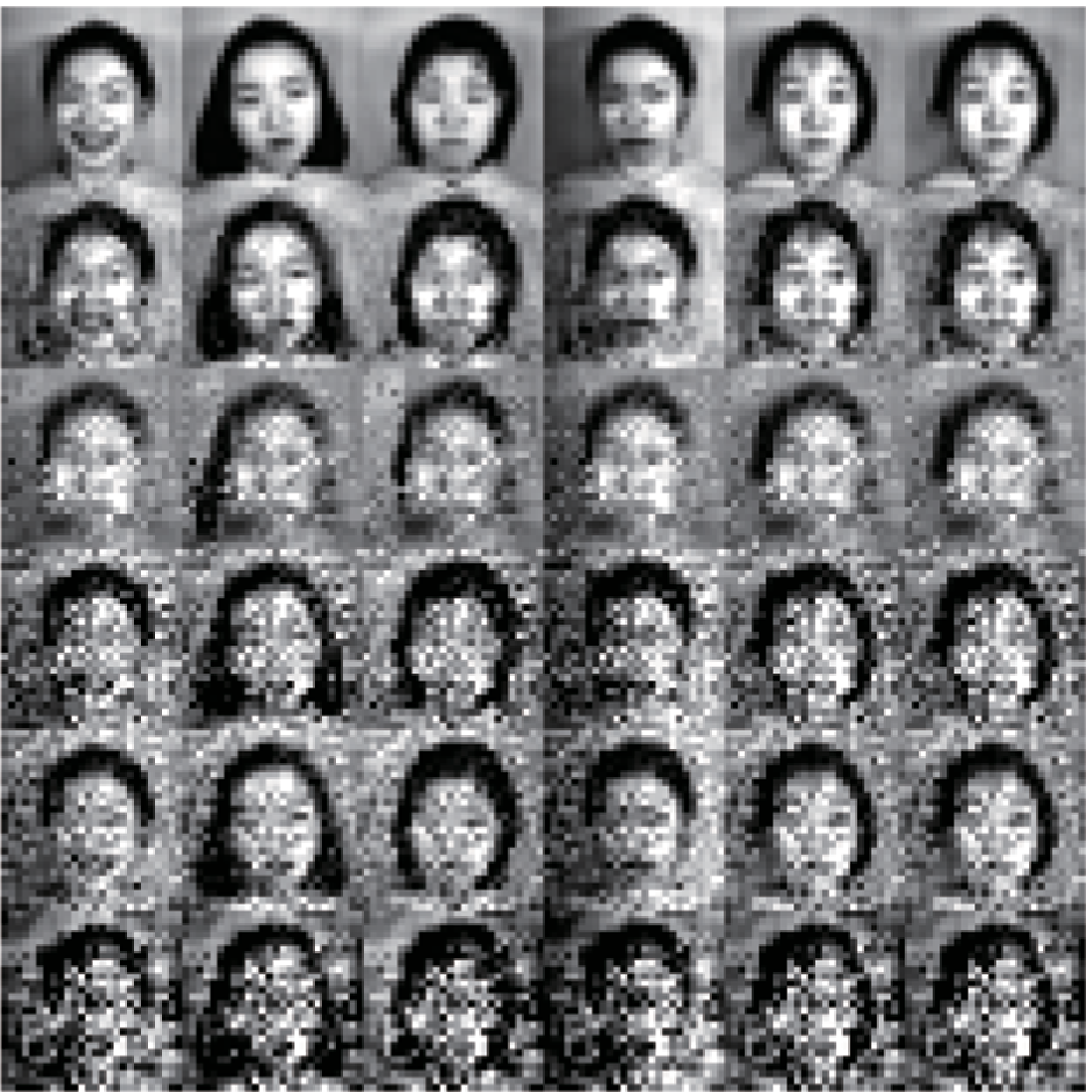} 
\caption{Reconstructed images of COIL20, ORL, and JAFFE datasets ($k=100, 200$ and $150$ respectively). The first row shows images in the test set. The second to the last row presents reconstructed images by SPCA (from left to right, $p=1.5, 1.0$ and $0.5$), $\ell_1$-norm based PCA, RPCA-OM, $\ell_{2,p}$-norm based RPCA (from left to right, $p=1.5, 1.0$ and $0.5$), RPCA-AOM, respectively.}
\label{rec_jaffe}
\end{figure*}

\begin{figure}[!htbp]
\centering
\includegraphics[width=0.45\textwidth]{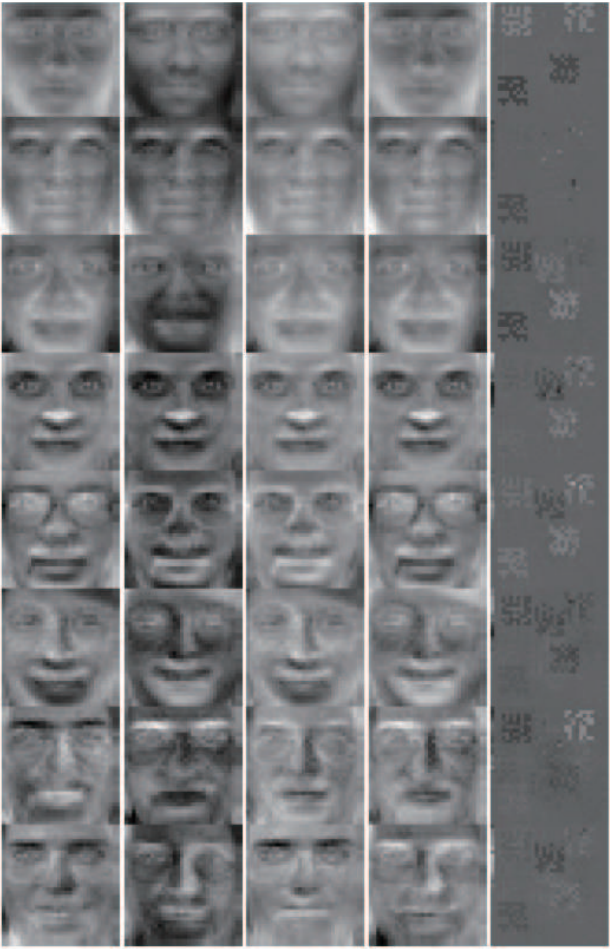} 
\caption{Some eigenfaces obtained on ORL. The left to the right column shows eigenfaces of SPCA, RPCA-OM, RPCA-AOM, $\ell_{2,p}$ based PCA, $\ell_1$ based PCA, respectively.}
\label{eigenface}
\end{figure}
 
\begin{figure}[!htbp]
     \centering
        \subfloat[$1^{st}$ on JAFFE iteration]{\includegraphics[width=.4\textwidth]{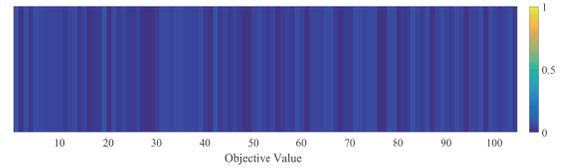}}\\
        \subfloat[$5^{th}$ on JAFFE iteration]{\includegraphics[width=.4\textwidth]{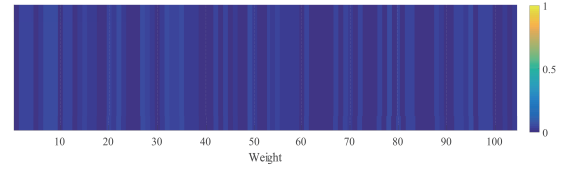}}\\
        \subfloat[$1^{st}$ on ORL iteration]{\includegraphics[width=.4\textwidth]{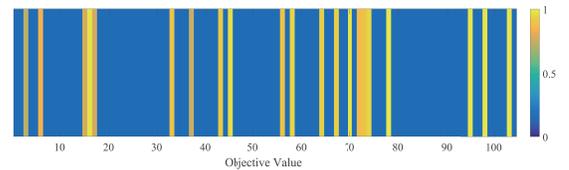}}\\
         \subfloat[$5^{th}$ on ORL iteration]{\includegraphics[width=.4\textwidth]{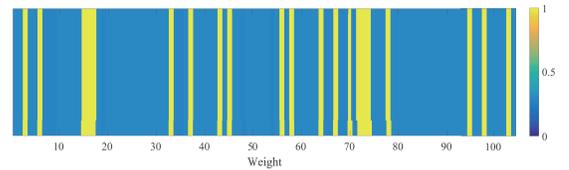}}\\

     \caption{ Visualization of the objective values and the weights at 1st and 5th iteration. The error values are normalized.}\label{cm}
\end{figure}

 \begin{figure}[!htbp]
\centering
\subfloat[Reconstruction Error of JAFFE]{\includegraphics[width=.4\textwidth]{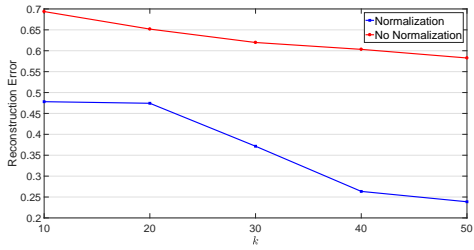}}\\
\subfloat[Reconstruction Error of ORL]{\includegraphics[width=.4\textwidth]{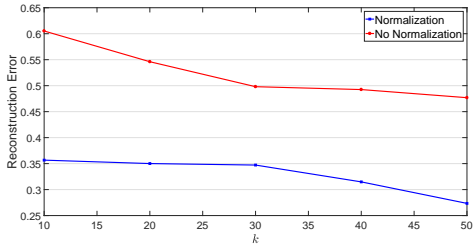}}\\
\caption{Reconstruction error with and without normalization on JAFFE and ORL databases.}
\label{Nor_jaffe}
\end{figure}

\begin{figure}[!htbp]
\centering
\subfloat[Reconstruction Error of JAFFE]{\includegraphics[width=.4\textwidth]{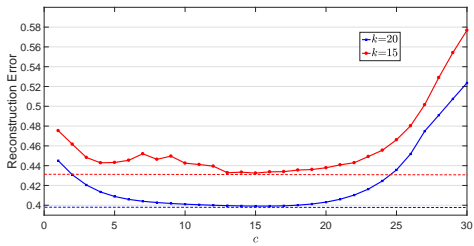}}\\
\subfloat[Reconstruction Error of ORL]{\includegraphics[width=.4\textwidth]{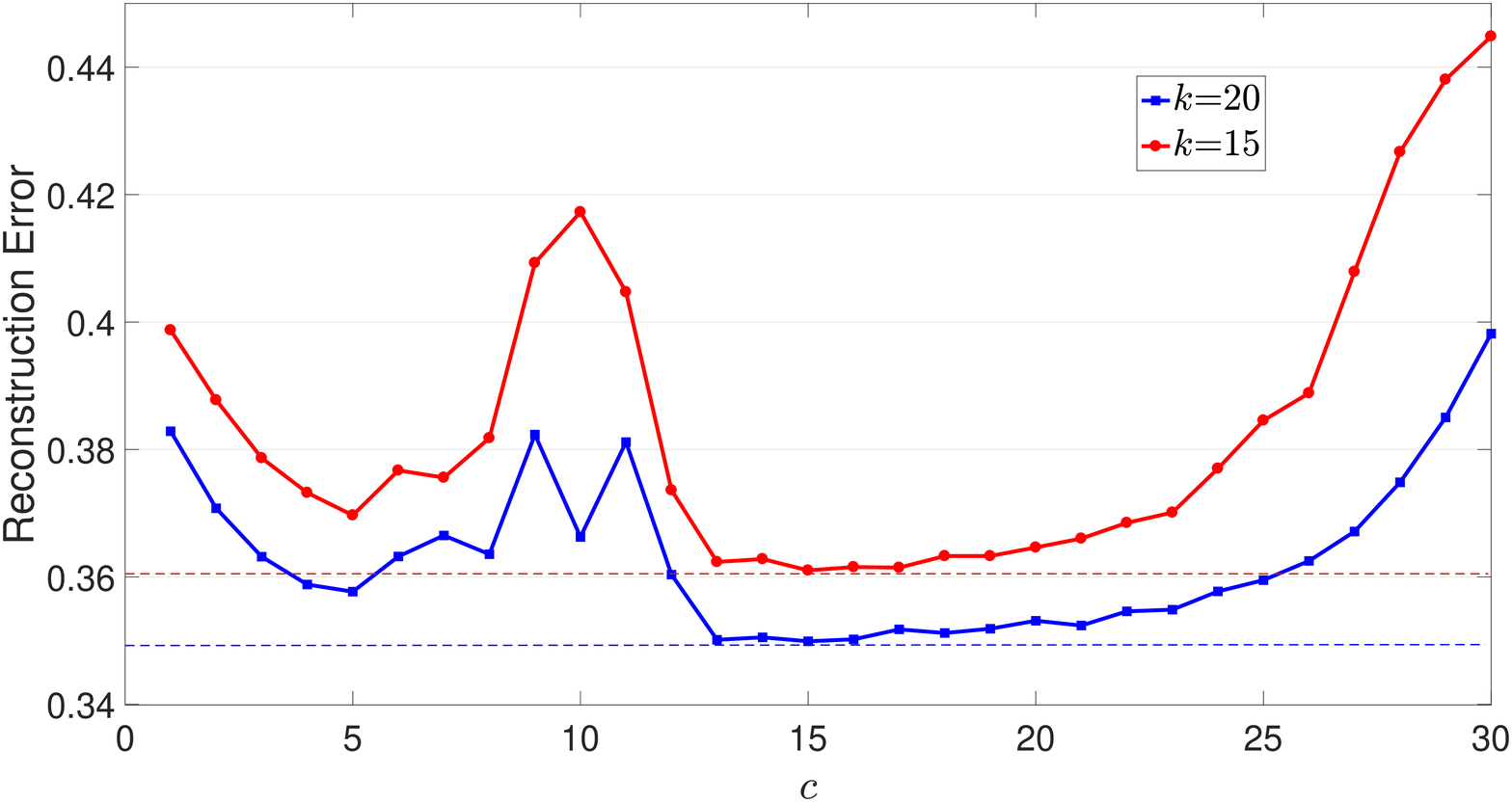}}\\
\caption{Given fixed $\eta$ = 0.1 and $p$ = 0.5, the impact of $c$ on reconstruction error for JAFFE and ORL data. The dashed line denotes the value of the lowest point.}
\label{c_ORL_Line}
\end{figure}

\begin{figure}[!htbp]
\centering
\subfloat[Reconstruction Error of JAFFE]{\includegraphics[width=.3\textwidth]{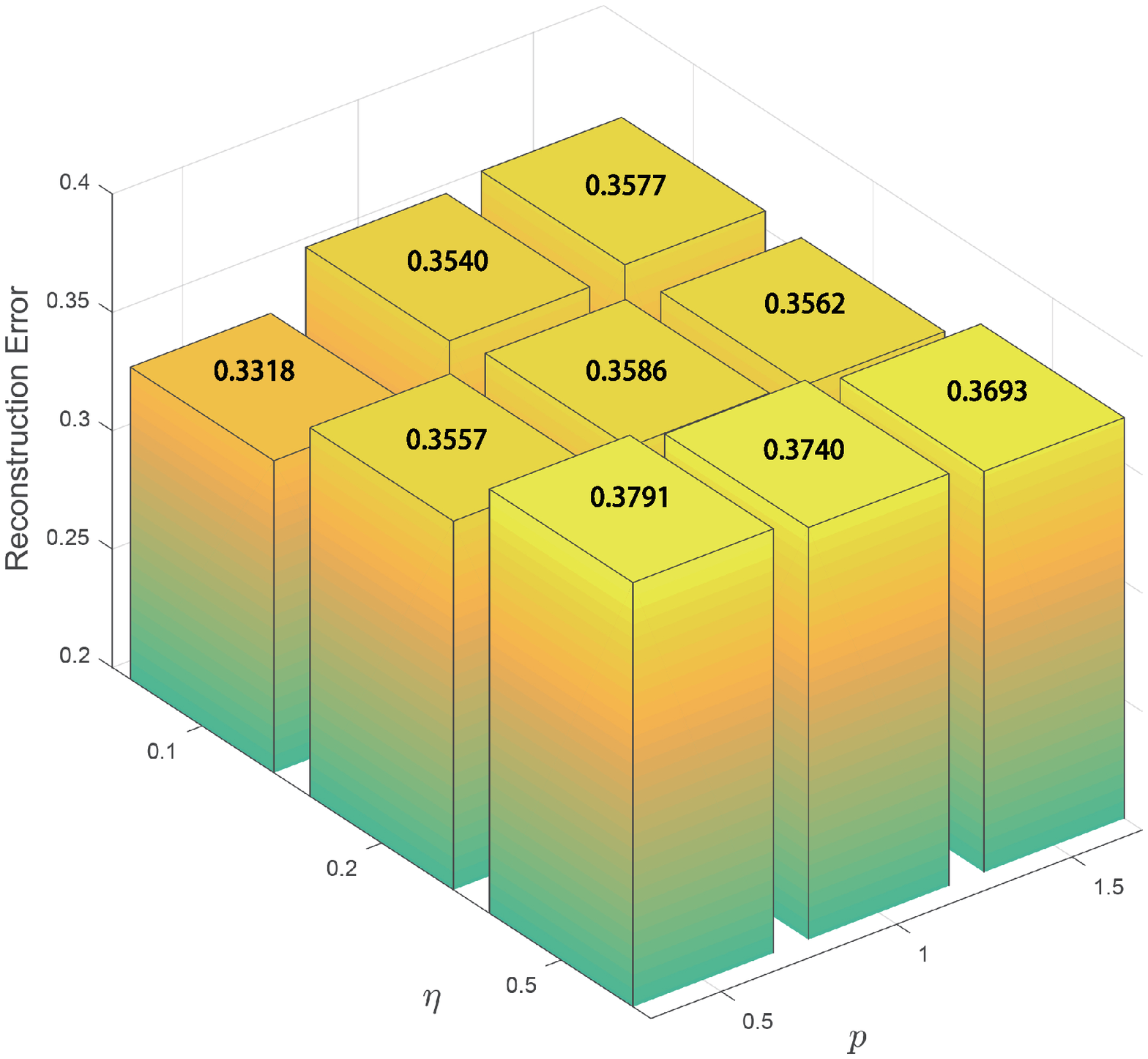}}\\
\subfloat[Reconstruction Error of ORL]{\includegraphics[width=.3\textwidth]{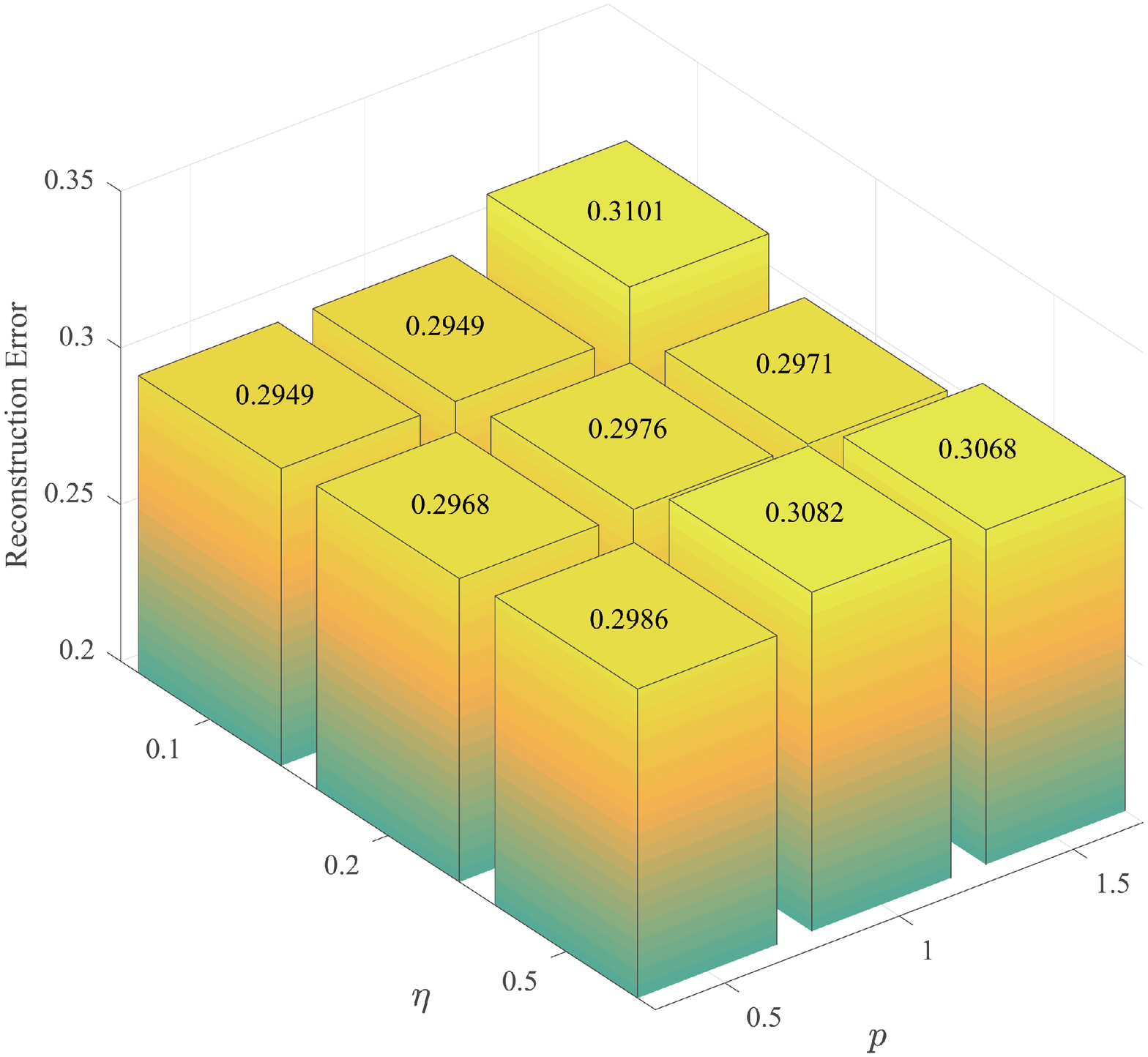}}\\
\caption{The impact of $\eta$ and $p$ on reconstruction error for JAFFE and ORL data. The error is averaged over different $k$ ($k$ = 10, 20, 30, 40, 50).}
\label{param}
\end{figure}

Table \ref{main_result} shows the reconstruction error of five methods with respect to various dimensions on three datasets. We can observe that our proposed SPCA significantly outperforms all other methods in all cases. In particular,
\begin{itemize}
\item $L_1$-PCA is overall inferior to the other four methods. The main reason is that it does not consider the mean drawback.
\item RPCA-AOM and RPCA-OM provide comparable performance. One possible reason is that RPCA-AOM gets stuck into bad local minimum.  
\item $L_{2,p}$-RPCA outperforms RPCA-OM and RPCA-AOM in most cases. This is attributed to the usage of $\ell_{2,p}$-norm which can suppress the effect of outliers. Other values of $p$ might be needed for $L_{2,p}$-RPCA to beat RPCA-OM and RPCA-AOM in all scenarios.
\end{itemize}

Figure \ref{rec_jaffe} presents part of the reconstructed images of COIL20, ORL and JAFFE datasets by five methods. It can be seen that the images reconstructed by our method are much better than other approaches. $L_1$-PCA is impacted by outliers to some degree. RPCA-OM, RPCA-AOM and $L_{2,p}$-RPCA can be easily influenced by outliers. In particular, they almost can not recover face images on ORL and JAFFE datasets, which is probably because face images are more complicated than objects in COIL20. The success of SPCA is due to our adoption of self-paced 
learning mechanism to filter out outliers, which leads to our projection vectors are less influenced by the outlying images. 

Taking ORL dataset as an example, we further compare the eigenfaces of five different methods in Figure \ref{eigenface}. It can be seen that most methods produce poor results. In particular, it is difficult to see any face in $L_1$-PCA. With respect to other methods, the eigenfaces of SPCA are less affected by the contaminated data.

We visualize the value of objective function and weights of JAFFE and ORL samples at the first and the fifth iteration in Figure \ref{cm}.
It can be seen that at the beginning of the training process, the weight of each sample is very small and close to zero. As the training progresses, the weight increases and the distinctions of complexity among samples are revealed.

\section{Parameter Analysis}

To better distinguish different samples, we apply Eq. \eqref{normalloss} to normalize the fidelity value of each sample to a specific interval. Figure \ref{Nor_jaffe} illustrates the difference between the reconstruction error with and without normalization. It demonstrates that normalization plays a crucial role. Furthermore, we show the influence of normalizing coefficient $c$'s value in Figure \ref{c_ORL_Line}, taking JAFFE and ORL dataset as examples. It indicates that SPCA has better performance when $c=15$. Figure \ref{param} presents the combination effect of $\eta$ and $p$. It illustrates that SPCA has better performance when both $\eta$ and $p$ are small and the reconstruction error reaches the minimum when $\eta=0.1$ and $p=0.5$.

\section{Conclusion}
In this paper, we build a more robust PCA formulation for dimensionality reduction by introducing the self-paced learning mechanism. To take sample difficulty levels into consideration, a novel regularizer is proposed to define the complexity of samples and control the learning pace. Subsequently, an alternative optimization strategy is developed to solve our model. Theoretical analysis is provided to reveal the convergence and robustness of our algorithm. Extensive experiments on three widely used datasets demonstrate the superiority of our method.



%
%
%
%
%
%
%
%
%
%

\bibliographystyle{IEEEtran}
\bibliography{ref}
\end{document}